\documentclass[wcp]{jmlrmod} 

\title{Online Isotonic Regression}
\usepackage{times}
\usepackage{booktabs}
\usepackage{commath,stmaryrd}
\usepackage{wrapfig}

\usepackage[textwidth=2.75cm,shadow,color=yellow]{todonotes}

\DeclareRobustCommand{\VAN}[3]{#2} 

\author{%
  \Name{Wojciech Kot{\l}owski} \Email{wkotlowski@cs.put.poznan.pl}\\
  \addr Pozna{\'n} University of Technology
  \AND
  \Name{Wouter M. Koolen} \Email{wmkoolen@cwi.nl}\\
  \addr Centrum Wiskunde \& Informatica
  \AND
  \Name{Alan Malek} \Email{malek@berkeley.edu}\\
  \addr University of California at Berkeley
}

\let\top\intercal
\newcommand{\adv}{i}
\newcommand{\hy}{\widehat{y}}
\newcommand{\reals}{\mathbb{R}}
\newcommand{\regret}{\mathrm{Reg}}
\newcommand{\isotonic}{\mathcal{F}}
\newcommand{\f}{\boldsymbol{f}}
\newcommand{\p}{\boldsymbol{p}}
\newcommand{\oomega}{\boldsymbol{\omega}}
\newcommand{\w}{\boldsymbol{w}}
\newcommand{\x}{\boldsymbol{x}}
\newcommand{\e}{\boldsymbol{e}}
\newcommand{\assert}[1]{\llbracket #1 \rrbracket}

\newcommand{\argmin}{\operatornamewithlimits{argmin}}
\newcommand{\argmax}{\operatornamewithlimits{argmax}}

\begin{document}

\maketitle

\makeatletter
\renewcommand{\@shortauthor}{Kot{\l}owski, Koolen and Malek}
\makeatother

\begin{abstract}
We consider the online version of the isotonic regression problem. Given a set of linearly ordered points (e.g., on the real line),
the learner must predict labels sequentially at adversarially chosen positions and is evaluated by her total squared loss compared against the best isotonic (non-decreasing) function in hindsight. We survey several standard online learning algorithms and show that none of them achieve the optimal regret exponent; in fact, most of them (including Online Gradient
Descent, Follow the Leader and Exponential Weights) incur linear regret. We then prove that the Exponential Weights algorithm played over a covering net of isotonic functions has a regret bounded by $O\big(T^{1/3} \log^{2/3}(T)\big)$ and present a matching $\Omega(T^{1/3})$ 
lower bound on regret. We provide a computationally efficient version of this algorithm. We also analyze the noise-free case, in which the revealed labels are isotonic, and show that the bound can be improved to $O(\log T)$ or even to $O(1)$ (when the labels are revealed in isotonic order). Finally, we extend the analysis beyond squared loss and give  bounds for entropic loss and absolute loss.
\end{abstract}


\section{Introduction}
We propose a problem of sequential prediction in the class of isotonic (non-decreasing) functions. At the start of the game, the learner is given a set of $T$ linearly ordered points (e.g., on the real line). Then, over the course of $T$ trials, the adversary picks a new (as of yet unlabeled) point and the learner predicts a label from $[0,1]$ for that point. Then, the true label
(also from $[0,1]$) is revealed, and the learner suffers the squared error loss. After $T$ rounds the learner is evaluated by means of the regret, which is its total squared loss minus the loss of the best isotonic function in hindsight.

Our problem is precisely the online version of \emph{isotonic regression}, a fundamental problem in statistics, which concerns fitting a sequence of data where the prediction is an isotonic function of the covariate \citep{Ayer55,Brunk55,isotonicbook}. Isotonic constraints arise naturally in many structured problems, e.g.\ predicting the height of children as a function of age,  autocorrelation functions, or biomedical applications such as estimating drug dose responses \citep{stylianou2002dose}. 
Despite being simple and commonly used in practice, isotonic regression is an example of nonparametric regression where the number of parameters 
 grows linearly with the number of data points. A natural question to ask is whether there are efficient, provably low regret algorithms for online isotonic regression. 

Since online isotonic regression concerns minimizing a convex loss function over the convex set of feasible prediction strategies (isotonic functions), it can be analyzed within the framework of online convex optimization \citep{ShalevShwartzBook}. We begin by surveying popular online learning algorithms in our setting and showing that most of them (including Online Gradient Descent, Follow the Leader and Exponential Weights) suffer regret that is linear in the number of data points in the worst case. The failure of most standard approaches makes the problem particularly interesting. We also show that the Exponentiated Gradient algorithm delivers a $O(\sqrt{T \log T})$ regret guarantee which is nontrivial but suboptimal.

We then propose an algorithm which achieves the regret bound $O\big(T^{1/3} \log^{2/3}(T)\big)$. The algorithm is a simple instance of Exponential Weights that plays on a covering net (discretization) of the class of isotonic functions. Despite the exponential size of the covering net, we present a computationally efficient implementation with $O(T^{4/3})$ time per trial. We also show a lower bound $\Omega(T^{1/3})$ on the regret of any algorithm, hence proving that the
proposed algorithm is optimal (up to a logarithmic factor). 

We also analyze the noise-free case where the labels revealed by the adversary are isotonic and therefore the loss of the best isotonic function is $0$. We show that the achievable worst-case regret in this case scales only logarithmically in $T$. If we additionally assume that the labels are queried in isotonic order (from left to right), the achievable worst-case regret drops to $1$. In both cases, we are able to determine the minimax algorithm and the actual value of the minimax regret.

Finally, we go beyond the squared loss and adapt our discretized
Exponential Weights algorithm to logarithmic loss and get the same regret guarantee. We also consider isotonic regression with absolute loss and show that the minimax regret is of order $\tilde{O}(\sqrt{T})$ and is achieved, up to a logarithmic factor, by the Exponentiated Gradient algorithm.


\subsection{Related work}

Isotonic regression has been extensively studied in statistics starting from work by \citet{Ayer55,Brunk55}.
The excellent book by \citet{isotonicbook} provides a history of the subject and numerous references to the statistical literature.
Isotonic regression has applications throughout statistics (e.g.\ nonparametric regression, estimating monotone densities, parameter estimation and statistical tests under order constraints, multidimensional scaling, see \citealt{isotonicbook}) and to more practical problems in biology, medicine, psychology, etc.\ \citep{kruskal64,stylianou2002dose,Obozinski,Luss2012}. 

The classical problem of minimizing an isotonic function under squared loss (the offline counterpart of this paper) has usually been studied in statistics under a generative model  $y_i = f(x_i) + \epsilon_i$ with $f(x_i)$ being some isotonic function and $\epsilon_i$ being random i.i.d.\ noise variables \citep{vanDeGeer90,BirgeMassart,zhang2002}. It is known \citep[see, e.g.,][]{zhang2002} that the statistical risk of the isotonic regression function $\mathbb{E}[\frac{1}{T}\|\widehat{f} - f\|^2]$ converges at the rate of $O(T^{-2/3})$, where $T$ is the sample size. Interestingly, this matches (up to a logarithmic factor) our results on online isotonic regression, showing that the online version of the problem is not fundamentally harder.

In machine learning, isotonic regression is used to calibrate class probability estimates \citep{Zadrozny2002,Niculescu-MizilC05,MenonJVEO12,Narasimhan2013,Valentina}, for ROC analysis \citep{Fawcett2007},
for learning Generalized Linear Models and Single Index Models \citep{Kalai2009,Kakade2011}, for data cleaning
\citep{KotSlo09ICML} and for ranking \citep{Moon2010}. Recent work by \citet{Kyng_etal15} proposes
fast algorithms under general partial order constraints. None of these works are directly related to the subject of this paper. The one related problem we found is online learning with logarithmic loss for the class of monotone predictors as studied by \citet{CesaBianchiLugosi2001}, who give an upper bound on the minimax regret (the bound is not tight for our case).

We also note that the problem considered here falls into a general framework of online nonparametric regression. \citet{RakhlinSridharan2014} give nonconstructive upper and lower bound on the minimax regret, but using their bounds for a particular function class requires upper and lower bounds on its sequential entropy. In turn, our upper bound is achieved by an efficient algorithm, while the lower bound follows from a simple construction. \citet{GaillardG15} propose an algorithm, called Chaining Exponentially Weighted Average Forecaster, that is based on aggregation on two levels. On the first level, a multi-variable version of Exponentiated Gradient is used, while on the second level, the Exponential Weights algorithm is used. The combined algorithm works for any totally bounded (in terms of metric entropy) set of functions, which includes our case. It is, however, computationally inefficient in general (an efficient adaptation of the algorithm is given for the H{\"o}lder class of functions, to which our class of isotonic functions does not belong). In contrast, we achieve the optimal bound by using a simple and efficient Exponential Weights algorithm on a properly discretized version of our function class (interestingly,
\citet{GaillardG15} show that a general upper bound for Exponential Weights, which works for any totally bounded nonparametric class, is suboptimal). 

\section{Problem statement}
\label{sec:problem_statement}

Let  $x_1 \le x_2 \le \ldots \le x_T$, be a set of $T$ linearly ordered points (e.g., on the real line), denoted by $X$. We call a function $f \colon X \to \reals$ \emph{isotonic} (order-preserving) on $X$ if $f(x_i) \le f(x_j)$ for any $x_i \le x_j$. Given data $(y_1,x_1),\ldots,(y_T,x_T)$, the isotonic regression problem is to find an isotonic $f$ that minimizes $\sum_{t=1}^T (y_t - f(x_t))^2$, and the optimal such function is called the \emph{isotonic regression function}.

We consider the online version of the isotonic regression problem. 
The adversary chooses $X=\{x_1,\ldots,x_T\}$ which is given in advance to the learner. In each trial $t=1,\ldots,T$, the adversary picks a yet unlabeled point $x_{i_t}$, $i_t \in \{1,\ldots,T\}$ and the learner predicts with $\hy_{\adv_t} \in [0,1]$. Then, the actual label $y_{\adv_t} \in [0,1]$ is revealed, and the learner is penalized by the squared loss $(y_{\adv_t} - \hy_{\adv_t})^2$. Thus, the learner predicts at all points $x_1, \ldots x_T$ but in an adversarial order.

The goal of the learner is to have small regret, which is defined to be the difference of the cumulative loss and the cumulative  loss of the best isotonic function in hindsight:
\[
\regret_T := \sum_{t=1}^T (y_{i_t} - \hy_{i_t})^2 - \min_{\mathrm{isotonic}~f}~ \sum_{t=1}^T (y_{i_t} - f(x_{i_t}))^2.
\]
Note that neither the labels nor the learner's predictions are required to be isotonic on $X$. In what follows, we assume without loss of generality that $x_1 < x_2 < \ldots < x_T$, because equal consecutive points $x_{i} = x_{i+1}$ constrain the adversary ($f(x_i) = f(x_{i+1})$ for any function $f$) but not the learner.

\paragraph{Fixed-design.} We now argue that without showing $X$ to the learner in advance, the problem is hopeless; if the adversary can choose  $x_{i_t}$ online, any learning algorithm will suffer regret at least $\frac{1}{4}T$ (a linear regret implies very little learning is happening since playing randomly obtains linear regret). To see this, assume the adversary chooses $x_{i_1}=0$; given learner's prediction $\hy_{i_1}$, the adversary can choose $y_{i_1}\in\{0,1\}$ to cause loss at least $\frac{1}{4}$. Now, after playing round $t$, the adversary chooses  $x_{i_{t+1}} = x_{i_t} - 2^{-t}$ if $y_{i_t}=1$ or $x_{i_{t+1}} = x_{i_t} + 2^{-t}$ if $y_{i_t}=0$. This allows the adversary to set $y_{i_{t+1}}$ to any value and still respect isotonicity. Regardless of $\hy_{i_{t+1}}$, the adversary inflicts loss at least $\frac{1}{4}$. This guarantees that if $y_{i_t}=1$ then $x_{i_q} < x_{i_t}$ for all future points $q=t+1,\ldots,T$; similarly, if $y_{i_t}=0$ then $x_{i_q} > x_{i_t}$ for all $q > t$. Hence, the label assignment is always isotonic on $X$, and the loss of the best isotonic function in hindsight is $0$ (by choosing $f(x_i) = y_i$, $i=1,\ldots,T$) while the total loss of the learner is at least $\frac{1}{4}T$.

\bigskip\noindent
Thus, the learner needs to know $X$ in advance. On the other hand, the particular values $x_i \in X$ do not play any role in this problem; it is only the order on $X$ that matters. Thus, we may without loss of generality assume that $x_i = i$ and represent isotonic functions by vectors $\f = (f_1,\ldots,f_T)$, where $f_i := f(i)$. We denote by $\isotonic$ the set of all $[0,1]$-valued isotonic functions:
\[
 \isotonic = \{\f = (f_1,\ldots,f_T) \colon 0 \le f_1 \le f_2 \le \ldots \le f_T \le 1\}.
\]
Using this notation, the protocol for online isotonic regression is presented in Figure~\ref{fig:online_protocol}.  

We will use $\widehat{L}_T = \sum_{t=1}^T (y_t - \hy_t)^2$ to denote the total loss of the algorithm
and $L_T(\f) = \sum_{t=1}^T (y_t - f_t)^2$ to denote the total loss of the isotonic function $\f \in \isotonic$.
The regret of the algorithm can then be concisely expressed as $\regret_T = \widehat{L}_T - \min_{\f \in \isotonic} L_T(\f)$.

\begin{figure}[t]
\begin{center}
\begin{tabular}{|l@{\hspace{10pt}}l|}
\hline
\multicolumn{2}{|l|}{At trial $t=1\dots T$:}\\
 &Adversary chooses index $i_t$, such that $i_t \notin \{i_1, \ldots, i_{t-1}\}$. \\
 &Learner predicts $\hy_{\adv_t} \in [0,1]$. \\
 &Adversary reveals label $y_{\adv_t} \in [0,1]$. \\
 &Learner suffers squared loss $(y_{\adv_t} - \hy_{\adv_t})^2$. \\
\hline
\end{tabular}
\end{center}
\caption{Online protocol for isotonic regression.}
\label{fig:online_protocol}
\end{figure}
\paragraph{The offline solution.} The classic solution to the isotonic regression problem is computed by the \emph{Pool Adjacent Violators Algorithm} (PAVA) \citep{Ayer55}. 
The algorithm is based on the observation that if the labels of any two consecutive points $i,i+1$ violate isotonicity, then we must have $f^*_i = f^*_{i+1}$ in the optimal solution and we may merge both points to their average. This process repeats and terminates in at most $T$ steps with the optimal solution. Efficient $O(T)$ time implementations exist \citep{deLeeuw_etal09}. There are two important properties of the  isotonic regression function $\f^*$ that we will need later \citep{isotonicbook}:
\begin{enumerate}
 \item The function $\f^*$ is piecewise constant and thus its level sets partition $\{1,\ldots,T\}$.
 \item The value of $\f^*$ on any level set is equal to the weighted average of labels within that set. 
\end{enumerate}

\section{Blooper reel}
\label{sec:blooper_reel}
The online isotonic regression problem concerns minimizing a convex loss function over the convex class of isotonic functions. Hence, the problem can be analyzed with online convex optimization tools \citep{ShalevShwartzBook}. Unfortunately, we find that  most of the common online learning algorithms completely fail on the isotonic regression problem in the sense of  giving linear regret guarantees or, at best, suboptimal rates of $O(\sqrt{T})$; see Table~\ref{tbl:algorithms}. We believe that the fact that most standard approaches fail makes the considered problem particularly interesting and challenging.

\begin{table}[t]
\begin{center}
\begin{tabular}{l r r}
\toprule
Algorithm & General bound & Bound for online IR \\
\midrule
Online GD & $G_2 D_2 \sqrt{T}$ & $T$ \\ 
EG & $G_{\infty} D_1 \sqrt{T \log d}$ & $\sqrt{T \log T}$ \\
FTL & $G_2 D_2 d \log T$  & $T^2 \log T$ \\
Exponential Weights & $d \log T$ & $T \log T$ \\
\bottomrule
\end{tabular}
\end{center}
\caption{Comparison of general bounds as well as bounds specialized to online isotonic regression for various standard online learning algorithms. For general bounds, $d$ denotes the
dimension of the parameter vector 
(equal to $T$ for this problem), 
$G_p$
is the bound on the $L_p$-norm of the loss gradient, and 
$D_q$
is the bound on the $L_q$-norm of the parameter vector. Bounds for 
FTL and Exponential Weights exploit the fact that the square loss is $\frac{1}{2}$-exp-concave  \citep{book}. 
}
\label{tbl:algorithms}
\end{table}

In the usual formulation of online convex optimization, for trials  $t=1,\ldots,T$, the learner predicts with a parameter vector $\w_t \in \reals^d$, the adversary reveals a convex loss function $\ell_t$, and the learner suffers loss $\ell_t(\w_t)$. To cast our problem in this framework, we set the prediction of the learner at trial $t$ to $\hy_{i_t} = \w_t^\top \x_{i_t}$ and the loss to $\ell_t(\w_t) = (y_{i_t} - \w_t^\top \x_{i_t})^2$. There are two natural ways to parameterize $\w_t, \x_{i_t} \in \reals^d$:
\begin{enumerate}
\item  The learner predicts some $\f \in \isotonic$ and sets $\w = \f$. Then, $\x_{i}$ is the $i$-th unit vector (with $i$-th coordinate equal to $1$ and the remaining coordinates equal to $0$). Note that $\sup_{\w} \|\w\|_2 = \sqrt{T}$ and $\|\nabla \ell(\w)\|_2 \leq 2$ in this parameterization.
 \item The learner predicts some $\f\in\isotonic$ and sets $\w=(f_1-f_0,f_2-f_1,\ldots,f_{T+1}-f_T) \in \reals^{T+1}$, i.e.\ the vector of differences of  $\f$ (we used two dummy variables $f_0 = 0$ and $f_{T+1}=1$); then, $\x_i$ has the first $i$ coordinates equal to $1$ and the last $T-i$ coordinates equal to $0$. Note that $\|\w\|_1 = 1$, $\|\nabla \ell(\w)\|_{\infty} \leq 2$, but $\sup_{y,\w} \|\nabla \ell(\w)\|_2 = 2\sqrt{T}$.
\end{enumerate}

Table~\ref{tbl:algorithms} lists the general bounds and their specialization to online isotonic regression for several standard online learning algorithms: Online Gradient Descent (GD) \citep{Zinkevich}, Exponentiated Gradient (EG) \cite{eg} when applied to exp-concave losses (which include squared loss, see \citealt{book}), Follow the Leader\footnote{The Online Newton algorithm introduced by \citet{logarithmic} is equivalent to FTL for squared loss.}, and Exponential Weights \citep{logarithmic}. EG is assumed to be used in the second parameterization, while the bounds for the remaining algorithms apply to both parameterizations (since $G_2 D_2 = \Omega(\sqrt{T})$ in both cases).

EG is the only algorithm that provides a meaningful bound of order $O(\sqrt{T \log T})$, as shown in Appendix~\ref{sec:appendix_EG_bound}. All the other bounds are vacuous (linear in $T$ or worse). This fact does not completely rule out these algorithms since we do not know a priori whether their bounds are tight in the worst case for isotonic regression. Next we will exhibit sequences of outcomes that cause GD, FTL and Exponential Weights to incur linear regret.

\begin{theorem}
For any learning rate $\eta \geq 0$ and any initial parameter vector $\f_1 \in \isotonic$, 
the Online Gradient Descent algorithm, defined as
\[
 \f_t = \argmin_{\f \in \isotonic} \left\{\frac{1}{2} \| \f - \f_{t-1} \|^2 +  2\eta (f_{t-1,i_{t-1}} - y_{i_{t-1}})f_{i_{t-1}}  \right\},
\]
suffers at least $\frac{T}{4}$ regret in the worst case.
\label{thm:gd}
\end{theorem}
\begin{proof}
The adversary reveals the labels in isotonic order ($i_t = t$ for all $t$), and all the labels are zero. Then, $\ell_t(\f_t) = \ell_t(\f_1)$, and the total loss of the algorithm $\widehat{L}_T$ is equal to the loss of the initial parameter vector: $\widehat{L}_T = L_T(\f_1) = \sum_t f_{1,t}^2$.
This follows from the fact that $\f_t$ and $\f_{t-1}$ can only differ on the first $t-1$ coordinates ($f_{t,q} = f_{t-1,q}$ for $q \geq t$) so only the coordinates of the already labeled points are updated. To see this, note that the parameter update can be decomposed into the ``descent'' step $\widetilde{\f}_{t} = \f_{t-1} - 2\eta f_{t-1,t-1} \e_{t-1}$ (where $\e_i$ is the $i$-th unit vector), and the ``projection'' step $\f_t = \argmin_{\f \in \isotonic} \|\f - \widetilde{\f}_{t}\|^2$ (which is actually the isotonic regression problem). 
The descent step decreases $(t-1)$-th coordinate by some amount and leaves the remaining coordinates intact. Since $\f_{t-1}$ is isotonic, $\widetilde{f}_{t,t} \leq \ldots \leq \widetilde{f}_{t,T}$ and $\widetilde{f}_{t,q} \leq \widetilde{f}_{t,t}$ for all $q < t$. Hence, the projection step will only affect the first $t-1$ coordinates.

By symmetry, one can show that when the adversary reveals the labels in \emph{antitonic} order ($i_t = T-t+1$ for all $t$), and all the labels are $1$, then $\widehat{L}_T = \sum_t (1-f_{1,t})^2$. Since $f_{1,t}^2 + (1-f_{1,t})^2 \geq \frac{1}{2}$ for any $f_{1,t}$, the loss suffered by the algorithm on one of these sequences is at least $\frac{T}{4}$.
%
\end{proof}

\vspace*{-10pt}

\begin{theorem}
For any regularization parameter $\lambda > 0$ and any regularization center $\f_0 \in \isotonic$, 
the Follow the (Regularized) Leader algorithm defined as:
\[
 \f_t = \argmin_{\f \in \isotonic} \Big\{ \lambda \| \f - \f_0 \|^2 + \sum_{q=1}^{t-1} (f_{i_q} - y_{i_q})^2 \Big\},
\]
suffers at least $\frac{T}{4}$ regret in the worst case.
\label{thm:ftrl}
\end{theorem}
\begin{proof}
The proof uses exactly the same arguments as the proof of Theorem \ref{thm:gd}: If the adversary reveals labels equal to $0$ in isotonic order, or labels equal to $1$ in antitonic order, then $f_{t,t} = f_{0,t}$ for all $t$. This is because the constraints in the minimization problem are never active ($\argmin$ over $\f \in \reals^T$ returns an isotonic function).
\end{proof}

\vspace*{-10pt}

We used a regularized version of FTL in Theorem \ref{thm:ftrl}, because otherwise FTL does not give unique predictions for unlabeled points.

\begin{theorem}
The Exponential Weights algorithm defined as:
\[
 \f_t = \int_{\isotonic}  \f p_t(\f) \dif \mu(\f), \qquad \text{where} \quad p_t(\f) = \frac{e^{-\frac{1}{2} \sum_{q=1}^{t-1} (f_{i_q} - y_{i_q})^2}}{\int_{\isotonic} e^{-\frac{1}{2} \sum_{q=1}^{t-1} (f_{i_q} - y_{i_q})^2} \dif \mu(\f)},
\]
with $\mu$ being the uniform (Lebesgue) measure over $\isotonic$, suffers regret $\Omega(T)$ in the worst case.
\label{thm:exp_weights_uniform}
\end{theorem}
The proof of Theorem \ref{thm:exp_weights_uniform} is long and is deferred to Appendix \ref{sec:appendix_Exp_Weights_proof}.

\section{Optimal algorithm}
\label{sec:optimal_algorithm}
We have hopefully provided a convincing case that many of the standard online approaches do not work for online isotonic regression.
Is this section, we present an algorithm that does: Exponential Weights over a discretized version of $\mathcal F$. We show that it achieves $O(T^{1/3} (\log T)^{2/3})$ regret which matches, up to log factors, the $\Omega(T^{1/3})$ lower bound we prove in the next section.

The basic idea is to form a covering net of all isotonic functions by discretizing $\mathcal F$ with resolution $\frac{1}{K}$, to then play Exponential Weights on this covering net with a uniform prior, and to tune $K$ to get the best bound. We take as our covering net $\mathcal{F}_K \subset \mathcal{F}$ the set of isotonic functions which take values of the form $\frac{k}{K}$, $k=0,\ldots,K$, i.e.
\[
\mathcal{F}_K := \left\{\f \in \mathcal{F} \colon f_t = \frac{k_t}{K} \; \text{~for~some~} k_t \in \{0,\ldots,K\},\, k_1 \leq \ldots \leq k_T \right\}.
\]
Note that $\mathcal{F}_K$ is finite. In fact $|\mathcal F_K| = \binom{T+K}{K}$, since the enumeration of all isotonic function in $\mathcal F_K$ is equal to the number of ways to distribute the $K$ possible increments among bins $[0,1),\ldots,[T-1,T),[T,T+1)$. The first and last bin are to allow for isotonic functions that start and end at arbitrary values. It is a well known fact from combinatorics that there are $\binom{T+K}{K}$ ways to allocate $K$ items into $T+1$ bins, \citep[see, e.g.,][section 2.4]{detemple2014combinatorial}.

The algorithm we propose is the Exponential Weights algorithm over this covering net; at round $t$, each $\f$ in $\mathcal F_K$ is given weight $e^{-\frac{1}{2} \sum_{q=1}^{t-1} (f_{i_q} - y_{i_q})^2}$ and we play the weighted average of $f_{i_t}$. An efficient implementation is given in Algorithm~\ref{alg:efficient}.

\begin{theorem}
Using $K=\left\lceil\left(\frac{T}{4 \log(T+1)}\right)^{1/3}\right\rceil$, the regret of Exponential Weights with the uniform prior on the covering net $\mathcal F_K$ has regret bounded by:
  \begin{equation*}
 \regret_T\leq \frac{3}{2^{2/3}}T^{1/3} \left(\log(T+1)\right)^{2/3} +2 \log(T+1).
  \end{equation*}
\end{theorem}
\label{thm:exp_weights_discretization}
\begin{proof}
Due to exp-concavity of the squared loss, running Exponential Weights with $\eta = 1/2$ guarantees that: 
\[
\widehat{L}_T - \min_{\f \in \mathcal{F}_K} L_T(\f) \leq \frac{\log |\mathcal F_K|}{\eta} = 2 \log |\mathcal F_K|,
\]
\citep[see, e.g.,][Proposition 3.1]{book}. 

Let $\f^* = \argmin_{\f \in \mathcal{F}} L_T(\f)$ be the isotonic regression function. 
The regret is 
\begin{align*}
\mathrm{Reg} &= \widehat{L}_T - L_T(\f^*) \\
&= \widehat{L}_T - \min_{\f \in \mathcal{F}_K} L_T(\f)
 + \underbrace{\min_{\f \in \mathcal{F}_K} L_T(\f) - L_T(\f^*)}_{:=\Delta_K}.
\end{align*}
Let us start with bounding $\Delta_K$. Let $\f^+$ be a function obtained from $\f^*$ by rounding each value $f^*_t$ to the nearest number of the form $\frac{k_t}{K}$ for some $k_t \in \{0,\ldots,K\}$. It follows that $\f^+ \in \mathcal{F}_K$ and $\Delta_K\leq L_T(\f^+) - L_T(\f^*)$.  Using $\ell_t(x):=(y_t-x)^2$, we have
\begin{align}
\label{eq:bound_f_plus_f_star}
\ell_t(f^+_t) - \ell_t(f^*_t) &= (y_t - f^+_t)^2 - (y_t - f^*_t)^2 
= (f^+_t - f^*_t)(f^+_t + f^*_t - 2y_t). 
\end{align}
Let $\mathcal{T}_c = \{t \colon f^*_t = c\}$ be the level set of the isotonic regression function. It is known \citep[see also Section \ref{sec:problem_statement}]{isotonicbook} that (as long as $|\mathcal{T}_c| > 0$):
\begin{equation}
\label{eq:average_property_isotonic_regression}
\frac{1}{|\mathcal{T}_c|} \sum_{t \in \mathcal{T}_c} y_t = f^*_t = c,
\end{equation}
i.e., the isotonic regression function is equal to the average over all labels within each level set. Now, choose any level set $\mathcal{T}_c$ with $|\mathcal{T}_c| > 0$. Note that $\f^+$ is also constant on $\mathcal{T}_c$ and denote its value by $c^+$. Summing \eqref{eq:bound_f_plus_f_star} over $\mathcal{T}_c$ gives:
\begin{align*}
\sum_{t \in \mathcal{T}_c} \ell_t(f^+_t) - \ell_t(f^*_t) 
&= \sum_{t \in \mathcal{T}_c} (c^+ - c) (c^+ + c - 2y_t) \\
&= |\mathcal{T}_c| (c^+ - c) (c^+ + c) - 2 (c^+ - c) \sum_{t \in \mathcal{T}_c} y_t \\
(\text{from \eqref{eq:average_property_isotonic_regression}})\qquad &= |\mathcal{T}_c| (c^+ - c) (c^+ + c) - 2 |\mathcal{T}_c| (c^+ - c) c \\
&= |\mathcal{T}_c| (c^+ - c)^2 \\
&= \sum_{t \in \mathcal{T}_c} (f_t^+ - f_t^*)^2.
\end{align*}
Since for any $t$, $|f^+_t - f^*_t| \leq \frac{1}{2K}$, we can sum over the level sets  of $\f^*$ to bound $\Delta_K$:
\[
\Delta_K \leq L_T(\f^+) - L_T(\f^*) 
\leq
\sum_{t=1}^T \ell_t(f_t^+) - \ell_t(f_f^*)
= 
\sum_{t=1}^T (f_t^+ - f_t^*)^2
\leq
 \frac{T}{4K^2}.
\]
Combining these two bounds, we get:
\begin{equation*}
  \regret_T \leq 2 \log |\isotonic_K| + \frac{T}{4K^2} \leq 2K\log(T+1) + \frac{T}{4K^2},
\end{equation*}
where we used $|\isotonic_K| = \binom{T+K}{K} \leq (T+1)^K$.\footnote{$\binom{T+K}{K} = \frac{(T+1) \cdot \ldots (T+K)}{1 \cdot \ldots \cdot K}$; we get the bound by noticing that $\frac{T+k}{k} \leq T+1$ for $k \ge 1$.} Optimizing the bound over $K$ by setting the derivative to $0$ gives $K^* = \left(\frac{T}{4 \log(T+1)}\right)^{1/3}$. Taking $K = \lceil K^* \rceil$ and plugging it in into the bound gives:
\[
 \regret_T ~\leq~ 2 (K^* + 1) \log(T+1) + \frac{T}{4 (K^*)^2} ~\leq~ \frac{3}{2^{2/3}}T^{1/3} \left(\log(T+1)\right)^{2/3} +2 \log(T+1),
\]
where we used $K^* \leq K \leq K^* + 1$.
\end{proof}

We note that instead of predicting with weighted average over the discretized functions, one can make use of the fact that the squared loss is $2$-mixable and apply the prediction rule of the Aggregating Forecaster \citep[Section 3.6]{Vovk90,book}. This would let us run the algorithm with $\eta=2$ and improve the leading constant in the regret bound to $\frac{3}{4}$.

\paragraph{The importance of being discrete.}
Surprisingly, playing weighted averages over $\mathcal F$ does not work (Theorem~\ref{thm:exp_weights_uniform}), but playing over a covering net does. Indeed, the uniform prior exhibits wild behavior by concentrating all mass around the ``diagonal'' monotonic function with constant slope $1/T$, whereas the discretized version with the suggested tuning for $K$ still has non-negligible mass everywhere.

\paragraph{Comparison with online nonparametric regression.} We compare our approach to the work of \citet{RakhlinSridharan2014} and \citet{GaillardG15}, which provide general upper bounds on the minimax regret expressed by means of the sequential and metric entropies of the function class under study. It turns out that we can use our covering net to show that the metric entropy $\log \mathcal{N}_2(\beta,\mathcal{F},T)$, as well as the sequential entropy $\log \mathcal{N}_{\infty}(\beta,\mathcal{F},T)$, of the class of isotonic functions are bounded by $O(\beta^{-1} \log{T})$; this implies (by following the proof of Theorem 2 of \citealp{RakhlinSridharan2014}, and by Theorem 2 of \citealp{GaillardG15}) that the minimax regret is bounded by $O(T^{1/3} (\log T)^{2/3})$, which matches our result up to a constant. Note, however,that the bound of  \citet{RakhlinSridharan2014} is nonconstructive, while ours is achieved by an efficient algorithm. The bound of \citet{GaillardG15} follows from applying the Chaining Exponentially Weighted Average Forecaster, that is based on aggregation on two levels: On the first level a multi-variable version of Exponentiated Gradient is used, while on the second level the Exponential Weights algorithm is used. The algorithm is, however, computationally inefficient in general, and it is not clear whether an efficient adaptation to the class of isotonic functions can easily be constructed. In contrast, we achieve the optimal bound by using a simple and efficient Exponential Weights algorithm on a properly discretized version of our function class; the chaining step turns out to be unnecessary for the class of isotonic functions due to the averaging property (\ref{eq:average_property_isotonic_regression}) of the isotonic regression function.

\subsection{An Efficient implementation}
A na\"{\i}ve implementation of exponential averaging has an intractable complexity of $O(|\isotonic_k|)$ per round. Fortunately, one can use dynamic programming to derive an efficient implicit weight update that is able to predict in $O(T K)$ time per round for arbitrary prediction orders and $O(K)$ per round when predicting in isotonic order. See Algorithm~\ref{alg:efficient} for pseudocode.


\begin{algorithm2e}[t]
\label{alg:efficient}
 \KwIn{Game length $T$, discretization $K$}
 Initialize $\beta_s^j=1$ for all $s=1,\ldots,T$,~ $j=0,\ldots,K$\;
 \For{$t=1,\ldots,T$}{
   Receive $i_t$\;
     Initialize $w_1^k = 1$ and  $v_T^k = 1$ for all $k=0,\ldots,K$\;
     \For{$s=2,\ldots,\adv_t$}{
     $
       w^k_s \leftarrow \assert{k > 0} w^{k-1}_s + \beta_{s-1}^k w_{s-1}^k
     $ for all $k=0,\ldots,K$\;
   }
   \For{$s=T-1,\ldots,\adv_t$}{
     $
       v^k_s \leftarrow \assert{k < K} w^{k+1}_s + \beta_{s+1}^k v_{s+1}^k
     $ for all $k=K,\ldots,0$\;
   }
$
  \hat y_{\adv_t} 
  \leftarrow
  \frac{
    \sum_{k=0}^K \frac{k}{K} w_{\adv_t}^k v_{\adv_t}^k
  }{
    \sum_{k=0}^K w_{\adv_t}^k v_{\adv_t}^k
  }
$\;
   Receive $y_{i_t}$ and 
   update 
   $
   \beta_{i_t}^j = e^{- \frac{1}{2} (\frac{j}{K} - y_{i_t})^2}
   $ for all $j=0,\ldots,K$\;
 }
\caption{Efficient Exponential Weights on the covering net}
\end{algorithm2e}

Say we currently need to predict at $\adv_t$. We can compute the Exponential Weights prediction by dynamic programming: for each $k = 0, \ldots, K$, let
\[
w_s^k ~= \! \sum_{0 \le f_1\le \ldots \le f_s = \frac{k}{K}} \! e^{- \frac{1}{2} \sum_{q < t : i_q < s} (f_{i_q} - y_{i_q})^2}
\quad\text{ and }\quad
v_s^k ~= \! \sum_{\frac{k}{K} = f_s \le \ldots \le f_T \le 1} \! e^{- \frac{1}{2} \sum_{q < t : i_q > s} (f_{i_q} - y_{i_q})^2},
\]
so that the exponentially weighted average prediction is
\[
\hat y_{\adv_t}
~=~
\frac{
  \sum_{\f \in \mathcal F_K} f_{\adv_t} e^{- \frac{1}{2} \sum_{q<t} (f_{\adv_q} - y_{\adv_q})^2}
}{
  \sum_{\f \in \mathcal F_K} e^{- \frac{1}{2} \sum_{q<t} (f_{\adv_q} - y_{\adv_q})^2}
}
~=~ 
\frac{
  \sum_{k=0}^K \frac{k}{K} w_{\adv_t}^k v_{\adv_t}^k
}{
  \sum_{k=0}^K w_{\adv_t}^k v_{\adv_t}^k
}
.
\]
Now we can compute the $w_s^k$ in one sweep from $s=1$ to $s=\adv_t$ for all $k = 0, \ldots, K$. If we define $\beta_s^j ~=~ 
e^{- \frac{1}{2} (\frac{j}{K} - y_s)^2}$  if $s \in \{\adv_1, \ldots, \adv_{t-1}\}$ and 1 otherwise, we can calculate $w_s^k$ by starting with $w_1^k = \beta_1^k$ and then sweeping right:
\begin{align*}
w_{s+1}^k
&~=~
\sum_{0 \le f_1\le \ldots \le f_{s+1} = \frac{k}{K}} e^{- \frac{1}{2} \sum_{q< t : \adv_q \le s} (f_{\adv_q} - y_{\adv_q})^2}
\\
&~=~
\sum_{0 \le j \le k}
\beta_s^j
\sum_{0 \le f_1\le \ldots \le f_s = \frac{j}{K}} e^{- \frac{1}{2} \sum_{q< t : \adv_q < s} (f_{\adv_q} - y_{\adv_q})^2}
\\
&~=~
\sum_{0 \le j \le k}
\beta_s^j
w_s^j.
\end{align*}
The equations for $v_s^k$ are updated symmetrically right-to-left, which gives an $O(T K^2)$ per round algorithm. We can speed it up to $O(T K)$ by using
\[
w_{s+1}^{k+1}
~=~
\sum_{0 \le j \le k}
\beta_s^j
w_s^j
+
\beta_s^{k+1}
w_s^{k+1}
~=~
w_{s+1}^k
+
\beta_s^{k+1}
w_s^{k+1},
\]
and similarly for $v_{s+1}^{k+1}$.

\paragraph{Acceleration for predicting in isotonic order.}
When working in isotonic order (meaning $\adv_t = t$), we can speed up the computation to $O(K)$ per round (independent of $T$) by the following tricks. First, we do not need to spend work maintaining $v_t^k$ as they satisfy $v_t^k = \binom{T-t+K-k}{K-k}$. Moreover, between rounds $t-1$ and $t$ the $w_s^k$ do not change for $s < t$, and we only need to compute $w_t^k$ for all $k$, hence speeding up the computation to $O(K)$ per round. 

\section{Lower bound}

\begin{theorem}
\label{thm:lower_bound}
All algorithms must suffer
  \begin{equation*}
    \regret_T = \Omega(T^{1/3}).
  \end{equation*}
\end{theorem}
The full proof is given in Appendix \ref{sec:appendix_lower_bound}.

\begin{proof}{\textbf{(sketch)}}
We proceed by constructing the difficult sequence explicitly. 
Split the $T$ points $(1,\ldots,T)$ 
into $K$ consecutive segments $(1,\ldots,m), (m+1,\ldots,2m), \ldots, (m(K-1)+1,\ldots,T)$, 
where in each segment there are $m = \frac{T}{K}$ consecutive
points (for simplicity assume $T$ is divisible by $K$). 
Let $t \in k$ mean that $t$ is in the $k$-th segment, $k=1,\ldots,K$.
Now, suppose the adversary generates the labels i.i.d.\
with $y_t \sim \mathrm{Bernoulli}(p_k)$ when $t \in k$, and
$p_1  \leq \ldots \leq p_K$.
The total loss of the best isotonic function is then bounded above by the total loss
of the constant function equal to $p_k$ in each segment, hence the expected regret
of any algorithm can be lower-bounded by
$\mathbb{E}[\mathrm{Reg}_T] \geq \sum_{k=1}^K \mathbb{E} \Big[ \sum_{t \in k} (\widehat{y}_t - p_k)^2 \Big]$.
In each segment,
the adversary picks $p_k \in \{p_{k,0},p_{k,1}\}$, where $p_{k,0}=\frac{1}{4} + \frac{k-1}{2K}$ and $p_{k,1} = \frac{1}{4} + \frac{k}{2K}$,
which guarantees that for any choice of the the adversary $p_1 \leq \ldots \leq p_K$.
We then show that the expected regret can be lower-bounded by:
\[
\mathbb{E}[\mathrm{Reg}_T] = \sum_{k=1}^K \sum_{t \in k} \mathbb{E}[(\widehat{y}_t - p_k)^2] \geq  \frac{m}{4} \sum_{k=1}^K \mathbb{E}[(\widehat{p}_k - p_k)^2],
\]
where $\widehat{p}_k \in \{p_{k,0},p_{k,1}\}$ depends on the predictions $\{\hy_t\}_{t \in k}$ (and hence on the data), but not on the probabilities $p_k$. 
We use Assouad’s lemma \citep{festschrift,tsybakov} to bound the sum on the right-hand side:
\[
 \max_{p_1,\ldots,p_K \colon p_k \in \{p_{k,0},p_{k,1}\}} \sum_{k=1}^K \mathbb{E}[(\widehat{p}_k - p_k)^2] \geq \frac{1}{8K} \left(1-\frac{\sqrt{m}}{\sqrt{3}K}\right). 
\]
Using $m=\frac{T}{K}$ and tuning the number of segments to $K = \Theta(T^{1/3})$ to optimize the bound,
gives $\Omega(T^{1/3})$ lower bound on the worst-case regret.
\end{proof}

We note that an analogous lower bound of the form $\Omega(T^{-2/3})$
is known in the statistical literature on isotonic regression as a lower bound 
for the statistical risk $\mathbb{E}[\frac{1}{T}\|\widehat{f} - f\|^2]$ of any estimator $\widehat{f}$
in the fixed-design setup under the standard i.i.d.\ noise assumption
\citep[see][for a brief overview of the lower and upper bounds in this setting]{zhang2002}. 
This shows that the online version of the problem is not fundamentally harder (up to a logarithmic factor) than the batch (statistical) version. 

\section{Noise-free case}
\label{sec:noise-free}

In this section, we are concerned with a particular case of ``easy data'', when the labels revealed by the adversary are actually isotonic: $y_1 \leq y_2 \leq \ldots \leq y_T$,
so that the loss of the best isotonic function is zero. We show that the achievable worst-case regret in this case scales only logarithmically in $T$.
Furthermore, if we additionally assume that the labels are revealed in isotonic order, the achievable worst-case regret is bounded by $1$. 
Interestingly, we were able to determine the minimax algorithm, and the exact value of the minimax regret in both cases. Our findings are summarized in the two theorems below. The proofs and the minimax predictions are given in Appendix \ref{sec:appendix_noise_free}.

\begin{theorem}
\label{thm:noise-free-any-order}
Assume the labels revealed by the adversary are isotonic. Then, the regret of the minimax algorithm is bounded above by:
\[
 \regret_T \leq \frac{1}{4} \log_2(T+1).
\]
Furthermore, when $T=2^k-1$ for some positive integer $k$, any algorithm suffers regret at least $\frac{1}{4} \log_2(T+1)$.
\end{theorem}

\begin{theorem}
\label{thm:noise-free-isotonic-order}
Assume the labels are isotonic, and they are revealed in isotonic order ($i_t = t$ for all $t$). 
Then, the regret of the minimax algorithm is bounded above by:
\[
 \regret_T \leq \alpha_T \leq 1,
\]
where $\alpha_T$ is defined recursively as:
$\alpha_1 = \frac{1}{4}$ and $\alpha_t = \big(\frac{\alpha_{t-1} + 1}{2}\big)^2$.
Furthermore, any algorithm suffers regret at least $\alpha_T$.
\end{theorem}

Finally, we note that the logarithmic regret can also be obtained by using the Exponentiated Gradient algorithm with its
learning rate tuned for the noise-free case \citep[see Appendix \ref{sec:appendix_EG_bound} and][for details]{eg}.

\section{Other loss functions}
\label{sec:other_loss_functions}

We discuss extensions of the isotonic regression problem where the squared loss is replaced by the entropic loss and the absolute loss respectively.
\subsection{Entropic loss}

The entropic loss, defined for $y,\hy \in [0,1]$ by $\ell(y,\hy) = -y \log \hy - (1-y) \log(1-\hy)$, plays an important role in isotonic regression, as its minimization is equivalent to maximum likelihood estimation for Bernoulli distributions under isotonic constraints \citep{isotonicbook}. It is convenient to replace the entropic loss by the relative entropy $D_{\phi}(y\|\hy) = \phi(y) - \phi(\hy) - (y-\hy)^\top \phi'(\hy)$,
which is the Bregman divergence generated by $\phi(y) = -y \log y - (1-y) \log(1-y)$, the binary entropy.
A surprising fact in isotonic regression is that minimizing the sum of Bregman divergences $\sum_t D_{\phi}(y_t \| f_t)$ in the class
of isotonic functions $\f \in \isotonic$ leads to the same optimal solution, no matter what $\phi$ is: the isotonic regression function $\f^*$ \citep{isotonicbook}.

Since the entropic loss is $1$-exp-concave \cite[page 46]{book}, we may use the Exponential Weights algorithm on the discretized class of functions:
\[
 \isotonic_K = \left\{\f \in \isotonic \colon \forall t, f_t \in \{z_0,z_1,\ldots,z_K\} \right\}
\]
(we now use a non-uniform discretization $\{z_0,z_1,\ldots,z_K\}$, to be specified later). Following the steps of the proof of Theorem \ref{thm:exp_weights_discretization}, we obtain a regret bound:
\[
 \regret_T \leq \log {T+K \choose K} + L_T(\f^+) - L_T(\f^*),
\]
where $L_T(\f) = \sum_{t=1}^T D_{\phi}(y_t \| f_t)$, $\f^+$ is defined by: $f^+_t = \argmin_{z \in \{z_0,z_1,\ldots,z_K\}} D_{\phi}(f^*_t \| z)$,
 and we used the fact that the isotonic regression function $\f^*$ minimizes $L_T(\f)$ over $\isotonic$ (see above).
Let $\mathcal{T}_c = \{t \colon f^*_t = c\}$ be a non-empty level set of $\f^*$. Using the averaging property (\ref{eq:average_property_isotonic_regression}) of $\f^*$, and the fact that $\f^+$ is constant
on $\mathcal{T}_c$ (denote its value by $c^+$), we have:
\begin{align*}
\sum_{t \in \mathcal{T}_c} D_{\phi}(y_t \| f^+_t) - D_{\phi}(y_t \| f^*_t) 
&= \sum_{t \in \mathcal{T}_c} \phi(c) - \phi(c^+) - (y_t - c^+) \phi'(c^+) + (y_t - c) \phi(c)  \\
&= |\mathcal{T}_c| D_{\phi}(c \| c^+) +  (\phi'(c) - \phi'(c^+))\sum_{t \in \mathcal{T}_c} (y_t - c)  \\
(\text{from \eqref{eq:average_property_isotonic_regression}})\qquad &=  |\mathcal{T}_c| D_{\phi}(c \| c^+) \\
&= \sum_{t \in \mathcal{T}_c} D_{\phi}(f_t^* \| f_t^+).
\end{align*}
Summing over the level sets gives $L_T(\f^+) - L_T(\f^*) = \sum_t D_{\phi}(f_t^* \| f_t^+)$. 
To introduce the appropriate discretization points, we follow \citep{discretization}. For any $y \in [0,1]$
and $\psi \in [0,\pi/2]$, we let $\psi(y) = \arcsin \sqrt{y}$, so that $y = \sin^2(\psi)$. The parameterization $\psi$ has a nice property,
that the values of $D_\phi$ on uniformly located neighboring points are also close to uniform.
We discretize the interval $[0,\pi/2]$ into $K+1$ points $\{\psi_0,\ldots,\psi_K\} = \big\{\frac{\pi}{2K},\ldots,\frac{\pi(K-1)}{2K}\big\} \cup \big\{\frac{\pi}{4K}, \frac{\pi}{2}-\frac{\pi}{4K} \big\}$, which
is almost uniform, with two additional points on the boundaries. Then, we define $z_k = y(\psi_k) = \sin^2(\psi_k)$, $k=0,\ldots,K$. Using Lemma 4 from \citet{discretization}:
\[
D_{\phi}(f_t^* \| f^+_t) \leq  \frac{(2 - \sqrt{2})\pi^2}{K^2},
\]
which bounds $L_T(\f^+) - L_T(\f^*)$ by $(2 - \sqrt{2})\pi^2 \frac{T}{K^2}$. From now on we proceed as in the proof of Theorem \ref{thm:exp_weights_discretization}, to get $O(T^{1/3} \log^{2/3}(T))$ bound. 
Thus, we showed:
\begin{theorem}
Using $K=\left\lceil\left(\frac{2(2-\sqrt{2})\pi^2T}{\log(T+1)}\right)^{1/3}\right\rceil$, the entropic loss regret of discretized Exponential Wights on the covering net:
\[
 \isotonic_K = \left\{\f \in \isotonic \colon \forall t, f_t \in \{z_0,z_1,\ldots,z_K\} \right\},
\]
where $z_0 = \sin^2(\frac{\pi}{4K}), z_K = \cos^2(\frac{\pi}{4K})$, and $z_k = \sin^2(\frac{\pi k}{2K})$ for $k=1,\ldots,K-1$,
has the following upper bound:
  \begin{equation*}
 \regret_T\leq \frac{3(2-\sqrt{2})^{1/3}\pi^{2/3}}{2^{2/3}}T^{1/3} \left(\log(T+1)\right)^{2/3} + 2\log(T+1).
  \end{equation*}
\end{theorem}

\subsection{Absolute loss}

Absolute loss $|\hy_{i_t} - y_{i_t}|$ is a popular loss function in modeling data with isotonic functions, especially
in the context of isotonic discrimination/classification \citep{DykstraHewettRobertson99,KotSlo2013TKDE}.
However, the online version of this problem turns out to be rather uninteresting for us, since
it can be solved in an essentially optimal way (up to an $O(\sqrt{\log T})$ factor) by using the vanilla Exponentiated Gradient algorithm. 
Applying the standard EG regret bound \citep[also c.f.\ Section \ref{sec:blooper_reel}]{eg,book} results in a $O(\sqrt{T \log T})$ bound,
whereas a lower bound of $\Omega(\sqrt{T})$ comes from the setting of prediction with expert advice \citep{book}: we constrain
the adversary to only play with one of the two constant (isotonic) functions $f_t \equiv 0$ or $f_t \equiv 1$, and apply
the standard lower bound for the 2-experts case.

\section{Conclusions and open problem}
\label{sec:conclusions}

We introduced the online version of the isotonic regression problem, in which the learner must sequentially predict the labels as well as the best isotonic function. We gave a computationally efficient version
of the Exponential Weights algorithm which plays on a covering net for the set of isotonic functions and proved that
its regret is bounded by $O(T^{1/3} \log^{2/3}(T))$.
We also showed an $\Omega(T^{1/3})$ lower bound on the regret of any algorithm, essentially closing the gap.

There are some interesting directions for future research. First, we believe that the discretization (covering net)
is not needed in the algorithm, and a carefully devised continuous prior would work as well. We were, however, unable to
find  a prior that would produce the optimal regret bound and remain computationally efficient.
Second, we are interested to see whether some regularized version of FTL (e.g., by means of relative entropy), or the \emph{forward
algorithm} (Vovk-Azoury-Warmuth) \citep{AzouryW01} could work for this problem.
However, the most interesting research direction is the extension to the partial order case. In this setting, the learner is given a set of points $X=\{x_1,\ldots,x_T\}$,
together with a partial order relation $\preceq$ on $X$. 
The goal of the learner is to sequentially predict the labels not much worse than the best function which respects the 
isotonic constraints: $x_i \preceq x_j \to f(x_i) \leq f(x_j)$.
A typical application would be nonparametric data modeling with multiple features, where domain knowledge may tell us that increasing the value of any of the features is likely to increase the value of the label. 
The off-line counterpart has been extensively studied in the statistics literature \citep{isotonicbook}, and the optimal
isotonic function shares many properties (e.g., averaging within level sets) with the linear order case.
The discretized Exponential Weights algorithm, which was presented in this paper, can be extended to deal with partial orders.
The analysis closely follows the proof of Theorem \ref{thm:exp_weights_discretization} except that the size of the covering
net $\isotonic_K$ is no longer $O(T^K)$ but now depends on the structure of $\preceq$. We believe that $|\isotonic_K|$ is the right
quantity to measure the complexity of the problem and the algorithm will remain competitive in this more general setting.
Unfortunately, the algorithm is no longer efficiently implementable and suffers from the same problems that plague inference in graphical models on general graphs. It thus remains an open problem to find an efficient algorithm for the partial order case.


\acks{
We thank the anonymous reviewers for suggestions which improved the quality of our work.
Wouter Koolen acknowledges support from the Netherlands Organization for Scientific Research (NWO, Veni grant 639.021.439),
Wojciech Kot{\l}owski acknowledges support from the Polish National Science Centre (grant no. 2013/11/D/ST6/03050)},
and Alan Malek acknowledges support from Adobe through a Digital Marketing Research Award.

\DeclareRobustCommand{\VAN}[3]{#3} 
\bibliography{isotonic_colt}
\DeclareRobustCommand{\VAN}[3]{#2} 

\appendix

\section{The Exponentiated Gradient (EG) bound}
\label{sec:appendix_EG_bound}

We will first cast the online isotonic regression problem to the equivalent problem of minimizing square loss over $(T+1)$-dimensional probability simplex $\Delta^{T+1}$.

Given $\f \in \isotonic$, define the $(T+1)$-dimensional vector of increments of $\f$ by $\p(\f) = (f_1 - f_0,f_2 - f_1,\ldots,f_{T+1} - f_T)$, where we used two dummy variables $f_0 = 0$ and $f_{T+1} = 1$.
Note that $\p(\f) \in \Delta^{T+1}$, and there is one-to-one mapping between elements from
$\isotonic$ and the corresponding elements from $\Delta^{T+1}$, with the inverse mapping $\f(\p)$ given by $f_t(\p) = \sum_{q=1}^t p_t$.
The loss in the simplex parameterization is given by:
\[
 \ell_t(\p_t) = \Big(y_{i_t} - \sum_{j \leq i_t} p_{t,j}\Big)^2
 = \Big(y_{i_t} - \p_t^\top \x_{i_t} \Big)^2,
\]
where $\x_{i_t}$ is the vector with the first $i_t$ coordinates equal to $1$. 
%
The Exponentiated Gradient (EG) algorithm \citep{eg} is defined through the update:
\[
 p_{t,j} = \frac{ p_{t-1,j} e^{- \eta \left(\nabla \ell_{t-1}(\p_{t-1})\right)_j}}{\sum_{k=1}^{T+1} p_{t-1,k} e^{- \eta \left(\nabla \ell_{t-1}(\p_{t-1})\right)_k}},
\]
with $\p_0$ being some initial distribution. The prediction of the algorithm is then $\hy_{i_t} = \sum_{j \leq i_t} p_{t,j}$.
We now use the standard upper bound for the regret of EG:
\begin{theorem}[Theorem 5.10 by \citealt{eg}] Let $\{(\x_t,y_t)\}_{t=1}^T$ be a sequence of outcomes such that for all $t$, $\max_{i} x_{t,i} - \min_i x_{t,i} \leq R$.
For any $\p \in \Delta^{T+1}$ with $L_T(\p) \leq K$ and $D(\p \| \p_0) \leq D$ for some $\p_0$, the EG algorithm with initial distribution $\p_0$ and learning
rate $\eta$ tuned as:
\[
 \eta = \frac{2 \sqrt{D}}{R(\sqrt{2K} + R\sqrt{D})},
\]
have the following bound on its cumulative loss:
 \[
  \widehat{L}_T \leq L_T(\p) + \sqrt{2 KD} + \frac{R^2 D(\p \| \p_0)}{2}.
 \]
\end{theorem}

We apply this theorem to our problem with the sequence permuted by $(i_1,\ldots,i_t)$ and $R=1$. 
We choose $\p_0$ to be a uniform distribution on $\Delta^{T+1}$, which means $D(\p \| \p_0) \leq \log (T+1) = D$. We also use a crude bound on the loss of
best comparator $\p^* = \argmin_{\p} L_T(\p)$, $L_T(\p^*) \leq \frac{1}{4} T = K$ (this is because the loss of the best comparator is lower than the loss
of the constant function $\f$ equal to the arithmetic mean of the data). This suggests tuning the learning rate to:
\[
 \eta = \frac{2\sqrt{\log(T+1)}}{\sqrt{\frac{T}{2}} + \sqrt{\log(T+1)}},
\]
to get the following regret bound:
\[
 \regret_T \leq \sqrt{\frac{T \log(T+1)}{2}} + \frac{\log(T+1)}{2}.
\]

\section{Proof of Theorem~\ref{thm:exp_weights_uniform} (bound for the Exponential Weights algorithm)}
\label{sec:appendix_Exp_Weights_proof}

Let the adversary reveal the labels in isotonic order ($i_t = t$ for all $t$), and they are all equal to $0$. At trial $t$, the prediction of the algorithm $\hy_t = \f_t^\top \x_t = f_{t,t}$
is given by:
\[
 \hy_t = \int_{\isotonic} f_t p_t(\f) \dif f_1 \ldots \dif f_T, 
 \qquad \text{where} 
 \quad p_t(\f) = \frac{e^{-\frac{1}{2}\sum_{q < t} f_q^2}}{\underbrace{\int_{\isotonic} e^{-\frac{1}{2}\sum_{q < t} f_q^2} \dif f_1 \ldots \dif f_T}_{=Z}},
\]
We calculate the marginal distribution $p_t(f_t = z)$:
\begin{align*}
 p_t(f_t = z) &= \int_{0 \leq f_1 \leq f_{t-1} \leq z \leq f_{t+1} \ldots \leq f_T \leq 1} p_t(\f) \dif f_1 \ldots \dif f_{t-1} \dif f_{t+1} \ldots \dif f_T \\ 
&= \frac{1}{Z} \bigg( \int_{0 \leq f_1 \leq \ldots \leq f_{t-1} \leq z} e^{-\frac{1}{2}\sum_{q < t} f_q^2} \dif f_1 \ldots \dif f_{t-1} \bigg)
 \bigg( \int_{z \leq f_{t+1} \leq \ldots \leq f_T \leq 1} \dif f_{t+1} \ldots \dif f_T \bigg) \\
 &= \frac{1}{Z} G(z,t-1) \frac{(1-z)^{T-t}}{(T-t)!},
\end{align*}
where:
\[
 G(z,n) = \int_{0 \leq f_1 \leq \ldots \leq f_n \leq z} e^{-\frac{1}{2} \sum_{t=1}^n f_t^2} \dif f_1 \ldots \dif f_n.
\]
We now calculate $G(z,n)$. Let $F(x) = \int e^{-\frac{1}{2}x^2} \dif x$ denote the antiderivative of the Gaussian.
Recursively applying the relation:
\[
 \int_{f_{t-1}}^z e^{-\frac{1}{2} f_t^2} \frac{1}{k!} (F(z) - F(f_t))^k
 = \frac{(F(z) - F(f_{t-1}))^{k+1}}{(k+1)!}.
\]
we get:
\[
G(z,n) = \frac{(F(z) - F(0))^n}{n!},
\]
so that:
\[
p_t(f_t = z) =  \frac{1}{Z'} (1-z)^{T-t}(F(z) - F(0))^{t-1}, \qquad \text{where} \quad Z' = Z (t-1)! (T-t)! \, .
\]
Denote $p_t(f_t=z)$ concisely as $\phi(z)$. Then, we have:
\[
\hy_t = \int_{0}^1 z \phi(z) \dif z. 
\]
Assume $t>1$ and let $\alpha = \frac{t-1}{T-1}$; note that $0 < \alpha \leq 1$. Define:
 \[
g(z) =  (1-\alpha) \log(1-z) + \alpha \log(F(z) - F(0)).
\]
Note that $\phi(z) = \frac{1}{Z'} e^{(T-1)g(z)}$. We have:
\[
g'(z) = -\frac{1-\alpha}{1-z} + \frac{\alpha}{F(z)-F(0)} e^{-\frac{1}{2}z^2},
\]
and:
\[
g''(z) = -\frac{1-\alpha}{(1-z)^2} - \frac{\alpha}{(F(z)-F(0))^2} e^{-z^2}
- \frac{\alpha}{F(z)-F(0)} z e^{-\frac{1}{2}z^2} < 0.
\]
Thus, $g$ is (strictly) concave, which implies that $\phi$ is log-concave. Furthermore,
due to strict concavity of $F(z)$ for $z > 0$, we have:
$F(0) < F(z) - z e^{-\frac{1}{2}z^2}$ for $z > 0$, which implies:
\[
g'(z) < -\frac{1-\alpha}{1-z} + \frac{\alpha}{z} \qquad \text{for} \,\, z >0,
\]
so that $g'(\alpha) < 0$. On the other hand, also from the concavity of $F(z)$,
$F(z) \leq F(0) + z$, which together with $e^{-\frac{1}{2}z^2} \geq 1 - \frac{1}{2}z^2$ implies:
\[
 g'(z) \geq -\frac{1-\alpha}{1-z} + \frac{\alpha\left(1-\frac{1}{2}z^2\right)}{z}.
\]
This means that:
\[
 g'\left(\frac{\alpha}{2}\right) \geq -\frac{1-\alpha}{1-\frac{\alpha}{2}} + 2\left(1-\frac{1}{8}\alpha^2\right)
 = \frac{1}{1-\frac{\alpha}{2}} - \frac{1}{4}\alpha^2 \geq 1 > 0.
\]
Thus $g'(z)$ switches the sign between $\frac{\alpha}{2}$ and $\alpha$, which means 
that the (unique) maximizer $z^* = \argmax g(z)$ is in the range $\left(\frac{\alpha}{2}, \alpha \right)$.

We now use \citep[Proposition 5.2]{logconcave} which states that for log-concave densities, the density at the mean
is not much smaller than the density at the mode:
\[
 \frac{1}{\sqrt{3} e} \sup_{z} \phi(z) \leq \phi(\hy_t) \leq \sup_{z} \phi(z),
\]
which means that after taking logarithms, dividing by $T-1$ and using the definition of $z^*$,
\[
 g(z^*) - \frac{1}{T-1} (1 + \log \sqrt{3}) \leq g(\hy_t) \leq g(z^*).
\]
From concavity of $g$,
\[
 g(\hy_t) \leq g\left(\frac{\alpha}{2}\right) + g'\left(\frac{\alpha}{2}\right) \left(\hy_t - \frac{\alpha}{2}\right)
 \leq g(z^*) + g'\left(\frac{\alpha}{2}\right) \left(\hy_t - \frac{\alpha}{2}\right),
\]
which, together with $g'\left(\frac{\alpha}{2}\right) \geq 1$, implies:
\[
\hy_t \geq \frac{\alpha}{2} + \frac{g(\hy_t) - g(z^*) }{g'\left(\frac{\alpha}{2}\right)}
\geq \frac{\alpha}{2} - \frac{\frac{1}{T-1} (1 + \log \sqrt{3})}{g'\left(\frac{\alpha}{2}\right)}
\geq \frac{\alpha}{2} - \frac{1}{T-1} (1 + \log \sqrt{3}).
\]
Hence, $\hy_t \geq \frac{\alpha}{4}$ when:
\[
 \frac{\alpha}{2} - \frac{1}{T-1} (1 + \log \sqrt{3}) \geq \frac{\alpha}{4} \qquad \Longrightarrow \qquad
 T \geq 1 + \frac{4(1 + \log \sqrt{3})}{\alpha}.
\]
Note, that this holds when $\alpha \geq \frac{1}{2}$ and  when $T \geq 14$. 
Therefore, when $T \geq 14$, for all $\alpha \geq \frac{1}{2}$, we have $\hy_t \geq \frac{\alpha}{4} \geq \frac{1}{8}$,
which means $\ell_t(\f_t) = (\hy_t-0)^2 \geq \frac{1}{64}$.
Since $\alpha \geq \frac{1}{2}$ is implied by $t \geq \lfloor T/2 \rfloor + 1$, we conclude that when $T \geq 14$,
\[
 \regret_T = \widehat{L}_T - \min_{f \in \isotonic} L_T(\f) = \widehat{L}_T \geq \sum_{t = \left\lfloor \frac{T}{2} \right\rfloor + 1}^T \frac{1}{64} \geq \frac{T}{128}.
\]

\section{Full proof of Theorem~\ref{thm:lower_bound}}
\label{sec:appendix_lower_bound}
We proceed by constructing the difficult sequence explicitly. 
First, split the $T$ points $(1,\ldots,T)$ 
into $K$ consecutive segments $(1,\ldots,m), (m+1,\ldots,2m), \ldots, (m(K-1)+1,\ldots,T)$, 
where in each segment there are $m = \frac{T}{K}$ consecutive
points (assume for now that $T$ is divisible by $K$). 
Let $t \in k$ mean that $t$ is in the $k$-th segment, $k=1,\ldots,K$,
i.e.\ $t \in k$ whenever $k = \left\lceil \frac{t}{m} \right\rceil$.
Now, suppose the adversary picks a $K$-vector $\p = (p_1,\ldots,p_K) \in [0,1]^K$ such that
$p_1 \leq p_2 \leq \ldots \leq p_K$ and generates the labels
in isotonic order (any order would work as well) such that $y_t \sim \mathrm{Bernoulli}(p_k)$ when $t \in k$.
Let $\f_{\p} \in \mathcal{F}$ denote an isotonic function such that $f_t = p_k$ when $t \in k$. 
We lower bound the expected regret:
\begin{align*}
 \mathbb{E}_{\p}[\mathrm{Reg}_T]
 &= \mathbb{E}_{\p} \left[ \widehat{L}_T - \inf_{\f \in \mathcal{F}} L_T(\f) \right] \\
 &\geq \mathbb{E}_{\p} \left[ \widehat{L}_T - L_T(\f_{\p}) \right] \\
 &= \sum_{k=1}^K \mathbb{E}_{\p} \left[ \sum_{t \in k} (\widehat{y}_t - y_t)^2 - (p_k - y_t)^2 \right] \\
 &= \sum_{k=1}^K \mathbb{E}_{\p} \left[ \sum_{t \in k} (\widehat{y}_t - p_k)(\widehat{y}_t + p_k - 2y_t)^2 \right] \\
 &= \sum_{k=1}^K \mathbb{E}_{\p} \left[ \sum_{t \in k} (\widehat{y}_t - p_k)^2 \right],
\end{align*}
where the last equality is from $\mathbb{E}_{p_k}[y_t] = p_k$.
Now we assume the adversary picks $\p$ from the following set:
\[
 \mathcal{P} = \left\{\p = (p_1,\ldots,p_K) \colon p_k \in \left\{\frac{1}{4} + \frac{k-1}{2K}, \frac{1}{4} + \frac{k}{2K}\right\} \right\}.
\]
There are $2^K$ vectors in $\mathcal{P}$, all satisfying $p_1 \leq p_2 \leq \ldots \leq p_K$, and
$\frac{1}{4} \leq p_k \leq \frac{3}{4}$ for all $k$. For instance, when $K=2$,
$\mathcal{P} = \left\{(\frac{1}{4},\frac{1}{2}),(\frac{1}{4},\frac{3}{4}),(\frac{1}{2},\frac{1}{2}),
(\frac{1}{2},\frac{3}{4}) \right\}$.

Fix $k$ and denote $p_{k,0} = \frac{1}{4} + \frac{k-1}{2K}$ and $p_{k,1} = \frac{1}{4} + \frac{k}{2K}$,
i.e.\ $p_k \in \{p_{k,0},p_{k,1}\}$. Define:
\[
\widehat{p}_k = \argmin_{p_{k,i},i=0,1} \left\{ |\overline{y}_k - p_{k,i}|  \right\},
\]
where $\overline{y}_k = \frac{1}{m} \sum_{t \in k} \widehat{y}_t$.
We now show that:
\begin{equation}
\sum_{t \in k} (\widehat{y}_t - p_k)^2 \geq \frac{m}{4} (\widehat{p}_k - p_k)^2.
\label{eq:inequality_widehat_f_widehat_p}
\end{equation}
Without loss of generality, assume $p_k = p_{k,0}$. Then, if $\widehat{p}_k = p_{k,0}$, the inequality clearly holds
because the right-hand side is $0$. On the other hand, if $\widehat{p}_k = p_{k,1}$, then 
from the definition of $\widehat{p}_k$ we have $|\overline{y}_k - p_{k,1}| \leq |\overline{y}_k - p_{k,0}|$,
which means $\overline{y}_k \geq \frac{1}{2}(p_{k,0} + p_{k,1})$.
This implies:
\begin{align*}
\sum_{t \in k} (\widehat{y}_t - p_{k,0})^2
&= \sum_{t \in k} (\widehat{y}_t - \overline{y}_k + \overline{y}_k - p_{k,0})^2 \\
&= \bigg( \sum_{t \in k} (\widehat{y}_t - \overline{y}_k)^2 + 2(\widehat{y}_t - \overline{y}_k)(\overline{y}_k - p_{k,0})\bigg) + m(\overline{y}_k - p_{k,0})^2 \\
&= \bigg( \sum_{t \in k} (\widehat{y}_t - \overline{y}_k)^2 \bigg) + m (\overline{y}_k - p_{k,0})^2 \\
&\geq m (\overline{y}_k - p_{k,0})^2 \geq \frac{m}{4} (p_{k,1} - p_{k,0})^2.
\end{align*}
Thus, (\ref{eq:inequality_widehat_f_widehat_p}) holds. Note that $\widehat{p}_k$ depends on $\{\widehat{y}_t, t \in k\}$
(which in turn depend on the labels), but it does not depend on $p_k$. 
Hence, the worst-case regret can be lower bounded by:
\[
 \max_{y_1,\ldots,y_T} \mathrm{Reg}_T \geq \max_{\p \in \mathcal{P}} \sum_{k=1}^K \frac{m}{4} \mathbb{E}_{\p}[(\widehat{p}_k - p_k)^2].
\]
We will now use Assouad’s lemma \citep{festschrift,tsybakov} to bound the sum on the right-hand side. Let $\Omega = \{0,1\}^K$ be the set
of all $2^K$ binary sequences of length $K$. The sequences from $\Omega$ will index the distributions from $\mathcal{P}$ by denoting $\p_{\oomega}=(p_{1,\omega_1},\ldots,p_{K,\omega_K}) \in \mathcal{P}$
for any $\oomega \in \Omega$.
We also define $\widehat{\oomega} \in \Omega$ as $\p_{\widehat{\oomega}} = (\widehat{p}_1,\ldots,\widehat{p}_K)$, i.e.\ $p_{k,\widehat{\omega}_k}=\widehat{p}_k$ for any $k$.
In this notation:
\[
\max_{\p \in \mathcal{P}} \sum_{k=1}^K \frac{m}{4} \mathbb{E}_{\p}[(\widehat{p}_k - p_k)^2] 
\geq \max_{\oomega \in \Omega} \frac{m}{4} \mathbb{E}_{\p_{\oomega}}\|\p_{\oomega} - \p_{\widehat{\oomega}} \|^2
= \max_{\oomega \in \Omega} \frac{m}{4} \frac{1}{4K^2} \mathbb{E}_{\p_{\oomega}} \rho(\oomega, \widehat{\oomega}),
\]
where $\rho(\cdot,\cdot)$ denotes the Hamming distance between two binary sequences.
Using \citep[Theorem 2.12.ii]{tsybakov}:
\[
\min_{\widehat{\oomega} \in \Omega} \max_{\oomega \in \Omega}  \mathbb{E}_{\p_{\oomega}} \rho(\oomega, \widehat{\oomega}) 
\geq \frac{K}{2} \left(1-\max_{\oomega, \oomega' \in \Omega \colon \rho(\oomega,\oomega')=1} \mathrm{TV}(\p_{\oomega},\p_{\oomega'})\right),
\]
where $\mathrm{TV}(\cdot,\cdot)$ is the total variation distance between distributions over $(y_1,\ldots,y_T)$.
From Pinsker's inequality:
\[
 \mathrm{TV}^2(\p_{\oomega},\p_{\oomega'}) \leq \frac{1}{2} D(\p_{\oomega}\| \p_{\oomega'}) = \frac{1}{2} \sum_{k=1}^K D(p_{k,\omega_k} \| p_{k,\omega'_k}),
\]
where $D(\cdot \| \cdot)$ is the Kullback-Leibler divergence and we used the fact the $\p \in \mathcal{P}$ are product distributions over segments.
Since $\rho(\oomega,\oomega')=1$, $p_{k,\omega_k}=p_{k,\omega'_k}$ for all but one $k$, hence all but one terms in the sum disappear
and we have
$
\mathrm{TV}^2(\p_{\oomega},\p_{\oomega'}) \leq \frac{1}{2} D(p_{k,0} \| p_{k,1}) 
$
for some $k$. Using the Taylor approximation of the KL-divergence with respect to $p_{k,1}$ around $p_{k,0}$, we get:
\[
D(p_{k,0} \| p_{k,1}) = \frac{m}{2} \frac{(p_{k,1} - p_{k,0})^2}{\widetilde{p}(1-\widetilde{p})},
\]
where $\widetilde{p}$ is some convex combination of $p_{k,0}$ and $p_{k,1}$. Since $p_{k,0},p_{k,1} \in [\frac{1}{4},\frac{3}{4}]$,
$\frac{1}{\widetilde{p}(1-\widetilde{p})}$ is maximized for $\widetilde{p} \in \{\frac{1}{4},\frac{3}{4}\}$, so that:
\[
 \mathrm{TV}^2(p_{k,0},p_{k,1}) \leq \frac{4 m}{3} (p_{k,1} - p_{k,0})^2 = \frac{m}{3 K^2}.
\]
Plugging this into our bound gives:
\[
\min_{\widehat{\oomega} \in \Omega} \max_{\oomega \in \Omega}  \mathbb{E}_{\p_{\oomega}} \rho(\oomega, \widehat{\oomega})
\geq \frac{K}{2} \left(1-\frac{\sqrt{m}}{\sqrt{3}K}\right), 
\]
which implies:
\[
 \max_{y_1,\ldots,y_T} \mathrm{Reg}_T \geq \frac{m}{32K} \left(1-\frac{\sqrt{m}}{\sqrt{3}K}\right)
 = \frac{T}{32K^2} \left(1-\frac{\sqrt{T}}{\sqrt{3} K^{3/2}}\right),
\]
where we used the fact that $m=\frac{T}{K}$.
%
Choosing $K = c T^{1/3}$ for some $c > 1$ gives:
\[
 \max_{y_1,\ldots,y_T} \mathrm{Reg}_T \geq \frac{c^{3/2} - 3^{-1/2}}{32 c^{7/2}} T^{1/3}.
\]
Choosing any $c>3^{-1/3}$, $c = O(1)$, such that $K$ divides $T$ finishes the proof.

\section{Proofs for Section~\ref{sec:noise-free} (noise-free case)}
\label{sec:appendix_noise_free}

\subsection{Proof of Theorem~\ref{thm:noise-free-any-order} (arbitrary order of outcomes)}

We first give a sequence of outcomes such that when $T = 2^k - 1$ for some positive integer $k$, any algorithm will suffer
exactly $\frac{1}{4} \log_2(T+1)$ loss. The adversary picks a point in the middle of the range, $i_1 = 2^{k-1}$.
After the algorithm predicts, the adversary chooses $y_{i_1}$ to be $0$ or $1$, depending which of these two incurs more loss to the algorithm.
Hence, no matter what the algorithm predicts, the loss is at least $\frac{1}{4}$.
If $y_{i_1} = 0$, then $2^{k-1}-1$ points on the left-hand side of $y_{i_1}$
are labeled to $0$ in the next trials (which is required due to noise-free regime), and the algorithm will possibly suffer no loss on these points.
Then, the adversary repeats the above procedure of choosing the middle point on the remaining $2^{k-1}-1$ points on the right-hand
side of $y_{i_1}$. Analogously, when $y_{i_1}=1$, the points on the right-hand side are all labeled to $1$, and the adversary
recursively play on remaining the left-hand side points.
This procedure can be carried out $k$ times, until no more points remains. Hence, the total loss incurred by
the algorithm is at least $\frac{1}{4} k = \frac{1}{4} \log_2(n+1)$.

Next, we determine the upper bound on the value of the minimax regret:
\[
 V = \min_{\hy_{i_1}} \max_{y_{i_1}} \ldots \min_{y_{i_T}} \max_{y_{i_T}} \sum_{t=1}^T (y_{i_t} - \hy_{i_t})^2 \leq \frac{1}{4} \log_2 (T+1),
\]
where the labels are constrained to be isotonic, $y_1 \leq \ldots \leq y_T$. We will
get the predictions of the minimax algorithm as a by-product of the calculations. This implies that the minimax algorithm
suffers regret at most $\frac{1}{4} \log_2 (T+1)$, and the bound on $V$ is tight whenever $T=2^k-1$.

In the first trial, the adversary reveals outcome $i_1$, which splits the set of unknown labels into
two disjoint sets $(y_1,\ldots,y_{i_t - 1})$ and $(y_{i_t+1},\ldots,y_T)$. The minimax algorithm knows that
$0 \leq y_1, \ldots, y_{i_t-1} \leq y_{i_t}$ and $y_{i_t} \leq y_{i_t+1},\ldots,y_T \leq 1$ (due to noise-free case).
Then, in the future trials,
each of these sets of unknown consecutive labels 
will be recursively split into smaller sets. At any moment of time, for any set
of unknown labels $(y_i, \ldots, y_j)$ with $j \geq i$, we know that
$y_{i-1} \leq y_i, \ldots, y_j \leq y_{j+1}$, and $y_{i-1}$ and $y_{j+1}$ has already been revealed
(we use $y_0 = 0$ and $y_{T+1} = 1$). Hence, the minimax algorithm will play a separate minimax game
for each set of unknown labels 
$(y_i, \ldots, y_j)$, bounded in the range $[y_{i-1}, y_{j+1}]$. We use this observation as a basis
for the recursion. Let $V(u,v,n)$ denote the minimax value of the game for a set
of $n$ not yet revealed consecutive labels lower-bounded by $u$, and upper-bounded by $v$.
We get the recursion:
\begin{equation}
 V(u,v,n+1) = \max_{k \in \{0,\ldots,n\}} V(u,v,n+1,k),
\label{eq:recursion_V_k}
\end{equation}
where:
\[
 V(u,v,n+1,k) = \min_{\hy \in [u,v]} \max_{y \in [u,v]} \left\{
 (y - \hy)^2 + V(u,y,k) + V(y,v,n-k)
 \right\},
\]
which follows from the fact that first the adversary reveals $(k+1)$-th point,
then the algorithm predicts with $\hy$ for that point, and finally the outcome $y$ is revealed, while
the set is split into two sets of smaller size. The minimax regret can be read out from $V(0,1,T)$.
To start the recursion, we define $V(u,v,0) = 0$.

We now prove by induction on $n$ that:
\begin{equation}
 V(u, v, n) = \beta_n (v-u)^2,
 \label{eq:V_beta_n_guess}
\end{equation}
where $\beta_n$ is some coefficient independent of $u$ and $v$.
Assume $n+1$ unknown labels, lower-bounded by $u$,
and upper-bounded by $v$. We fix $k \in \{0,\ldots,n\}$
and calculate the optimal prediction of the algorithm for $V(u,v,n+1,k)$:
\[
\hy = \argmin_{\hy \in [u,v]} \max_{y \in [u,v]} \left\{(y - \hy)^2 + \beta_k (y - u)^2 + \beta_{n-k} (v - y)^2 \right\}, 
\]
where we used the inductive assumption. The function inside $\max$ is convex
in $y$, hence the solution w.r.t.\ $y$ is $y \in \{u,v\}$. First note that if $\beta_k - \beta_{n-k} > 1$,
the function inside $\max$ is increasing in $y$ for any choice of $\hy \in [u,v]$, hence the optimal
choice for the adversary is $y=v$, and the optimal choice for the algorithm is $\hy = v$. Similarly,
if $\beta_k - \beta_{n-k} < -1$, the function inside $\max$ is decreasing in $y$, which results in
the optimal choice $y=u$ and $\hy = u$. When $-1 \leq \beta_k - \beta_{n-k} \leq 1$, it is easy
to check that the optimal prediction is obtained by setting the function inside $\max$ equal for
both choices of $y \in \{u,v\}$. This gives:
\[
 \hy = \frac{u+v}{2} + \frac{u-v}{2} \left(\beta_k - \beta_{n-k} \right) \in [u,v].
\]
Thus, depending on the value of $\beta_k - \beta_{n-k}$, $V(u,v,n+1,k)$ is given by:
\[
V(u,v,n+1,k) = (u-v)^2 \beta_{n,k},
\]
where
\begin{equation}
 \beta_{n,k} = \left\{
 \begin{array}{ll}
  \beta_{k} & \qquad \text{if~~} \beta_{k} - \beta_{n-k} > 1, \\
 \frac{1}{4}(\beta_k -\beta_{n-k})^2 + \frac{1}{2}(\beta_k + \beta_{n-k}) + \frac{1}{4} & \qquad \text{if~~} -1 \leq \beta_{k} - \beta_{n-k} \leq 1, \\
  \beta_{n-k} & \qquad \text{if~~} \beta_{k} - \beta_{n-k} < -1.
 \end{array}
\right.
\label{eq:def_beta_n_k}
\end{equation}
From (\ref{eq:recursion_V_k}), we have:
\[
 \beta_{n+1} = \max_{k \in \{0,\ldots,n\}} \beta_{n,k},
\]
which proves our inductive hypothesis (\ref{eq:V_beta_n_guess}).

What is left to show is that $\beta_n \leq \frac{1}{4} \log_2(n+1)$. We will prove it by induction on $n$. For $n=0$,
$\beta_0 = 0$ and thus the bound trivially holds.
We now show that $\beta_{n,k}$, as defined in (\ref{eq:def_beta_n_k}), is nondecreasing in $\beta_k$ 
and $\beta_{n-k}$. We fix $\beta_{n-k}$, and calculate the derivative with respect to $\beta_k$:
\[
 \frac{\dif \beta_{n,k}}{\dif \beta_k} = \left\{
 \begin{array}{ll}
  1 & \qquad \text{if~~} \beta_{k} - \beta_{n-k} > 1, \\
 \frac{1}{2}(\beta_k -\beta_{n-k}+1) & \qquad \text{if~~} -1 \leq \beta_{k} - \beta_{n-k} \leq 1, \\
  0 & \qquad \text{if~~} \beta_{k} - \beta_{n-k} < -1,
 \end{array}
\right. 
\]
which is nonnegative. Hence, $\beta_{n,k}$ is nondecreasing with $\beta_k$ for any fixed $\beta_{n-k}$. An analogous arguments shows that
$\beta_{n,k}$ is nondecreasing with $\beta_{n-k}$ for any fixed $\beta_k$. Hence, we can replace $\beta_k$ and $\beta_{n-k}$ by their upper bounds from
the inductive argument,
and then $\beta_{n,k}$ (as well as $\beta_{n+1}$) will not decrease. Thus, to show that $\frac{1}{4} \log_2((n+1)+1)$ is the upper bound on $\beta_{n+1}$,
it suffices to show that for any $n,k$, $\beta_{n,k} \leq \frac{1}{4} \log_2(n+2)$ after substituting 
$\beta_k = \frac{1}{4} \log_2(k+1)$ and $\beta_{n-k} = \frac{1}{4} \log_2(n-k+1)$ in (\ref{eq:def_beta_n_k}).

We proceed by cases in (\ref{eq:def_beta_n_k}). When $\beta_k - \beta_{n-k} > 1$, $\beta_{n,k} = \beta_k = \frac{1}{4} \log_2(k+1)
\leq \frac{1}{4} \log_2(n+2)$, because $k \leq n$. Case $\beta_k - \beta_{n-k} < -1$ is covered in an analogous way. We are left with the case
$-1 \leq \beta_k - \beta_{n-k} \leq 1$. It suffices to show that:
\begin{equation}
 \underbrace{(\beta_k -\beta_{n-k})^2 + \frac{1}{2}(\beta_k + \beta_{n-k}) + \frac{1}{4}}_{=g(\beta_k,\beta_{n-k})} \leq
 \underbrace{\frac{1}{4}\log_2(2^{4\beta_k} + 2^{4 \beta_{n-k}})}_{=f(\beta_k,\beta_{n-k})},
\label{eq:beta_bound}
\end{equation}
because the right-hand side is equal to $\frac{1}{4}\log_2(n+2)$ when $\beta_k=\frac{1}{4} \log_2(1+k)$ and $\beta_{n-k}=\frac{1}{4} \log_2(1+n-k)$.
Assume w.l.o.g.\ that $\beta_k \leq \beta_{n-k}$ (because both $f(\cdot,\cdot)$ and $g(\cdot,\cdot)$ are symmetric in their arguments).
For any $\delta$, $g(x+\delta,y+\delta) = g(x,y) + \delta$, and similarly $f(x+\delta,y+\delta) = f(x,y) + \delta$.
Thus, proving $f(x,y) \geq g(x,y)$ is equivalent to proving $f(0,y-x) \geq g(0,y-x)$. 
Given the condition $-1 \leq \beta_{n-k} - \beta_k \leq 1$, we thus need to show that:
$f(0,y) \geq g(0,y)$ for any $0 \leq y \leq 1$, which translates to:
\[
 \log_2\left(1 + 2^{4y}\right) \geq (1+y)^2, \qquad \text{for~~} 0 \leq y \leq 1.
\]
This inequality can be shown
by splitting the range $[0,1]$ into $[0,\frac{1}{4}]$, $[\frac{1}{4},\frac{3}{4}]$ and $[\frac{3}{4},1]$,
lower-bounding the left-hand side by its Taylor expansion up to the second order around points $0$, $\frac{1}{2}$, and $\frac{3}{4}$,
respectively (with the second derivative replaced by its lower bound in a given range), and showing that the corresponding quadratic
inequality always holds within its range. We omit the details here. Unfortunately, we were unable to find a more elegant
proof of this inequality.

\subsection{Proof of Theorem~\ref{thm:noise-free-isotonic-order} (isotonic order of outcomes)}

We determine the value of the minimax regret:
\[
 V = \min_{\hy_1} \max_{y_1 \in [0,1]} \min_{\hy_1} \max_{y_2 \in [y_1,1]} \ldots \min_{\hy_T} \max_{y_T \in [y_{T-1},1]} \sum_{t=1}^T (y_t - \hy_t)^2,
\]
getting the predictions of the minimax algorithm as a by-product of the calculations. Note that any algorithm will suffer regret at least $V$ for some sequences of labels,
while the minimax algorithm will suffer regret at most $V$ for any sequence of labels.
Let:
\[
 V_{T-t}(y_t) =  \min_{\hy_{t+1}} \max_{y_{t+1} \in [y_t,1]} \ldots \min_{\hy_T} \max_{y_T \in [y_{T-1},1]}  \sum_{q=t+1}^T (y_q - \hy_q)^2
\]
be the value-to-go function, which is the worst-case loss suffered by the minimax algorithm in $T-t$ trials $t+1,\ldots,T$, given
the last revealed label was $y_t$. The minimax regret $V$ can be read out from $V_T(0)$. We use the following recursion:
\[
 V_n(c) = \min_{\hy} \max_{y \in [c,1]} \left\{ (y-\hy)^2 + V_{n-1}(y) \right\},
\]
where we used $V_0(y) = 0$.
We start with calculating $V_1(c)$ (which corresponds to the last trial $t=T$). The minimax prediction is given by:
\[
 \argmin_{\hy} \max_{y \in [c, 1]} (\hy - y)^2
 = \argmin_{\hy} \max \{(\hy-c)^2, (\hy-1)^2 \}
 = \frac{c + 1}{2},
\]
and the value-to-go function is $V_1(c) = \frac{1}{4} (1 - c)^2$.
We now prove by induction that $V_n(c) = \alpha_n (1 - c)^2$ for
some $\alpha_n > 0$ (which clearly holds for $n=1$ with $\alpha_1=\frac{1}{4}$, as shown above).
By the induction argument,
\[
V_n(c) = \min_{\hy} \max_{y \in [c,1]} \left\{ (\hy-y)^2 + \alpha_{n-1} (1-y)^2 \right\}
= \min_{\hy} \max_{y \in \{c,1\}} \left\{ (\hy-y)^2 + \alpha_{n-1} (1-y)^2 \right\}.
\]
The last equality is due to the fact that the function inside the $\min\max$ is convex in $y$, therefore the optimal $y$ is on the boundary of 
the feasible range $[c,1]$. It is easy to check that the optimal $\hy$ makes the expression inside $\min\max$ equal for both choices of $y$,
so that:
\[
 \hy_t = \frac{c+1}{2} + \alpha_{n-1} \frac{c-1}{2},
\]
The expression inside max is a convex function of $y_t$, therefore the optimal $y_t$ is on the boundary of feasible range $\{y_{t-1},b\}$.
The algorithm predicts with $\hy_t$, such that the expression inside max has the same value for $y_t = y_{t-1}$ and $y_t = b$.
This gives:
\[
 \hy_t = \frac{b+y_{t-1}}{2} - \alpha_{t+1}^T \frac{b-y_{t-1}}{2},
\]
and:
\[
 V_n(c) = (\hy - 1)^2 = \underbrace{\left(\frac{\alpha_{n-1} +1}{2} \right)^2}_{=\alpha_n} (1-c)^2.
\]
This finishes the inductive proof for $V_n(c)$. The value of the minimax regret is given by $V_T(0) = \alpha_T$. 
Now, given that $\alpha_1 = \frac{1}{4} < 1$, we have inductively for all $n$:
\[
 \alpha_n = \left(\frac{\alpha_{n-1} + 1}{2}\right)^2 \leq \left(\frac{1+1}{2}\right)^2 = 1,
\]
\end{document}